\documentclass{article}
\usepackage{ijcai22}

\usepackage{times}
\usepackage{soul}
\usepackage{url}
\usepackage[hidelinks]{hyperref}
\usepackage[utf8]{inputenc}
\usepackage[small]{caption}
\usepackage{graphicx}
\usepackage{amsmath}
\usepackage{amsthm}
\usepackage{booktabs}
\usepackage{algorithm}
\usepackage{algorithmic}

\usepackage{amssymb}
\newcounter{protocol}
\makeatletter
\newenvironment{protocol}[1][htb]{%
  \let\c@algorithm\c@protocol
  \renewcommand{\ALG@name}{Protocol}
  \begin{algorithm}[#1]%
  }{\end{algorithm}
}

\newtheorem{theorem}{Theorem}

\title{EFMVFL: An Efficient and Flexible Multi-party Vertical Federated Learning without a Third Party}

\author{
Yimin Huang$^1$
\and
Xinyu Feng$^1$\and
Wanwan Wang$^{1}$\and
Hao He$^1$\and
Yukun Wang$^1$\and
Ming Yao$^1$\footnote{Corresponding author.}
\affiliations
$^1$InsightOne Tech Co, Ltd
\emails
\{huangyimin, fengxinyu, wangwanwan, hehao, wangyukun, yaoming\}@insightone.cn
}

\begin{document}

\maketitle

\begin{abstract}
    Federated learning allows multiple participants to conduct joint modeling without disclosing their local data. Vertical federated learning (VFL) handles the situation where participants share the same ID space and different feature spaces. In most VFL frameworks, to protect the security and privacy of the participants' local data, a third party is needed to generate homomorphic encryption key pairs and perform decryption operations. In this way, the third party is granted the right to decrypt information related to model parameters. However, it isn't easy to find such a credible entity in the real world. Existing methods for solving this problem are either communication-intensive or unsuitable for multi-party scenarios. By combining secret sharing and homomorphic encryption, we propose a novel VFL framework without a third party called EFMVFL, which supports flexible expansion to multiple participants with low communication overhead and is applicable to generalized linear models. We give instantiations of our framework under logistic regression and Poisson regression. Theoretical analysis and experiments show that our framework is secure, more efficient, and easy to be extended to multiple participants.

\end{abstract}

\section{Introduction}
With the improvement of data security and privacy protection laws, enterprises worldwide have to face the dual needs of solving data silos problems and protecting data privacy.
Federated learning (FL) was proposed by \cite{googleFL1} for joint machine learning modeling of multi-party data, 
where original data will not be shared among participants, only the intermediate gradients or parameters are transmitted. 

FL was mainly applied in the scenario of horizontal distribution of data when it was first proposed \cite{HFL1,googleFL,HFL2}, which assumes that each
client's data share the same feature space, but few sample IDs are overlapped among multiple participants.
Then \cite{silosFL} extended the concept of FL to horizontal federated learning (HFL), vertical federated learning (VFL), and federated transfer learning (FTL). 
In this paper, we will focus on VFL, which is one of the most commonly used methods in FL among enterprises, 
since it can improve the effect of machine learning models in many scenarios such as medicine, finance, advertising, and marketing by expanding the feature dimension of 
the same samples.

However, it has been proved that the local data could be inferred through the exposing of intermediate results (e.g., gradients) in original FL \cite{Leakage}. 
In subsequent works, cryptographic techniques were introduced into FL, 
e.g., Secret Sharing (SS) \cite{SSFL}, Homomorphic Encryption (HE) \cite{HEFL}, Differential privacy (DP) \cite{DPFL}, etc. 
HE offers an elegant way to protect the intermediate results by allowing encrypted data to be blindly processed
and has been widely used in privacy-preserving generalized linear models (e.g., linear Regression~\cite{silosFL}, logistic Regression (LR)~\cite{HEPR2017,HELR2018,ZWW21}), 
tree-based models (e.g., Decision Tree~\cite{VFLtree}, eXtreme Gradient Boosting (XGB)~\cite{xgb}, Random Forest~\cite{RF}) and deep learning models~\cite{DNN1,DNN2,DNN3}.

In most HE-based works, a third party is often needed to assist in generating HE key pairs and decrypting ciphertexts.
In this way, the third party is able to get the plaintext information related to model parameters. 
So the third party must be fair and credible and can not collude with other parties. Such a credible entity is difficult to find in practice.

A series of VFL frameworks without a third party were proposed to handle this problem. However, there are still some problems. 
For example, some methods allow both participants to generate HE key pairs and send the public key to each other for encrypting the data that needs to be exchanged during training process.
When the final result needs to be decrypted, the party will add noises generated by itself before sending it to the other party for decryption, 
and the party can restore accurate information by removing the noises.
Those frameworks only support two participants initially and need taking a lot of work to expand to three or more participants.
In fact, many methods with a third party~\cite{silosFL,HEPR2017,HELR2018} are not easy to be expanded to multiple parties either because of using a partially homomorphic encryption.

Another solution to the above problem is to use Secure Multi-Party Computation (MPC), such as SS-based methods~\cite{SSFL,SSLR2020,SSPR2021}. Those methods can support multi-party modeling.
While during the training process, 
model weights and all of the training data needs to be shared through a secret-sharing algorithm, which introduces an enormous communication overhead. 

\cite{ZWW21} proposed a framework that combines SS and HE together in privacy-preserving VFL logistic regression. It solved the problem that features are usually
high-dimensional and sparse in practice because of missing feature values or feature engineering such as one-hot encoding; meanwhile, it removed the third party.
As it shared the model parameters besides intermediate results following the idea of MPC,
it is hard for this framework to be extended to multiple parties, and it suffers from more communication overheads than our framework.

In this paper, we proposed an Eﬀicient and Flexible Multi-party Vertical Federated Learning without a Third Party called EFMVFL by combining HE and SS. 
Instead of sharing model parameters, 
we still follow the idea of FL that the model weights corresponding to the features are kept locally and updated locally by the respective participants.
Only necessary intermediate results (such as the inner product of feature data and model weights) are shared. And then, we can train the model based on SS data
and use HE to protect the privacy. 
In this way, our framework is scaling to multiple participants flexibly with low communication overhead and applicable to generalized linear models.

\subsection*{Our Contributions}
\begin{enumerate}
    \item We propose a new VFL framework combining SS and HE,
    where no trusted third party is needed.
    \item Our framework is flexible to be applied in generalized linear models.
    \item Our framework is easy to be extended to support multiple parties.
    \item We implement our framework under logistic regression (LR) and Poisson regression (PR), and prove that our VFL framework is more efficient and 
         has a lower communication overhead compared with recent popular works.
\end{enumerate}

\section{Notations}
We first introduce some notations here. As mentioned above, the focused scenario of our framework is VFL, in which data are
vertically partitioned by parties, and there is only one party holding the label.
Furtherly, we use $\mathbf{C}$ to  denote the party with the label (also named data demander), 
and use $\mathbf{B_i}$ to denote parties without label (also named data provider), where $\mathbf{i}$ is from $1$ to the num of data providers.

Correspondingly, in federated learning models, $Y$ denotes the label of party $\mathbf{C}$, $X_{p}$ denotes the features, 
$W_{p}$ denotes the linear model coefficients, $W_{p}X_{p}$ denotes the dot between model coefficients $W_{p}$ and $X_{p}$,
$\{pk_p,sk_p\}$ denotes the HE key pairs, $\left[\left[x\right]\right]_{p}$ denotes the ciphertext of x that are encrypted by using $pk_{p}$,
where $p$ is the party that can be $\mathbf{C}$ or $\mathbf{B_i}$, $\langle x\rangle _{p}$ is the secret share of x in party $\mathbf{P}$.

\section{Preliminaries}
\subsection{Secret-Sharing-based MPC}
Secure Multi-Party Computation (MPC) was first proposed in \cite{MPC}, which allows multiple parties to jointly perform privacy-preserving computing tasks.
Secret Sharing \cite{SS} is one of the main techniques to achieve this goal by splitting local data into shares first. 
Only secret-sharing addition and multiplication operations are used in our framework.

Assuming that party $\mathbf{C}$ holds data $X_c$ and party $\mathbf{B_1}$ holds data $X_{b_1}$. 
Firstly, they both need to share their data using Protocol \ref{alg:ss-protocol} introduced in the next section.
Then they can calculate the shares of $X_c+X_{b_1}$ and $X_c*X_{b_1}$ as follows:
\begin{enumerate}
\item\textbf{Addition}: The share of $X_c+X_{b_1}$ in party $\mathbf{C}$ (i.e., $\langle X_c+X_{b_1} \rangle_c$) can be calculated by $\langle X_c\rangle_c+\langle X_{b_1}\rangle_c$, 
similarly, $\langle X_c+X_{b_1} \rangle_{b_1} = \langle X_c\rangle_{b_1}+\langle X_{b_1}\rangle_{b_1} $
\item\textbf{Multiplication}: Besides the secret sharing of the original data, additional Beaver's triplet ($(\mu, \nu, \omega)$ and $\omega=\mu * \nu$ ) \cite{Beaver91} is required for calculating shares of $X_c*X_{b_1}$.
Specifically, each party get the share that $\langle X_c*X_{b_1} \rangle = \langle \omega \rangle + (\langle X_c\rangle - \langle \mu \rangle)*\langle \nu \rangle+(\langle X_{b_1}\rangle - \langle \nu \rangle)*\langle \mu \rangle+(\langle X_c\rangle-\langle \mu \rangle)*(\langle X_{b_1}\rangle -\langle \nu \rangle$).
\end{enumerate}
There have been many protocols to complete the above calculations, such as secureML \cite{SecureML}, ABY3\cite{ABY3}, secureNN \cite{SecureNN}, SPDZ \cite{SPDZ}, etc.

\subsection{Homomorphic Encryption}
HE methods support computation over ciphertexts, and the result of operating on ciphertexts and then decrypting is the same with the mathematical operations directly on plaintexts. 
As in secret sharing, only the addition and multiplication are needed in our framework. Specifically, the use of HE to calculate $X_c+X_{b_1}$ and $X_c*X_{b_1}$ mainly consists of the following steps:
\begin{enumerate}
    \item\textbf{Key generation}: One party (e.g., $\mathbf{C}$ ) generates HE key pairs ($\{pk_c,sk_c\}=\textbf{Gen}(1^\lambda)$, where $\lambda$ is a security parameter), and can send public key $pk_c$ to the other party.
    \item\textbf{Encryption}: $\mathbf{C}$ uses the $pk_c$ to encrypt data ($\left[\left[ X_c \right] \right]_{c}=\textbf{Enc}(X_c,pk_c)$). In the same way, the other party (i.e. $\mathbf{B_1}$) can get $\left[\left[ X_{b_1} \right] \right]_{c}$.
    \item\textbf{Addition}: Given $\left[\left[ X_c \right] \right]_{c}$ and $\left[\left[ X_{b_1} \right] \right]_{c}$, the addition between the two ciphertexts is $\left[\left[ X_c \right] \right]_{c} \bigoplus  \left[\left[ X_{b_1} \right] \right]_{c} = \left[\left[ X_c + X_{b_1} \right] \right]_{c}$.
    \item\textbf{Multiplication}: Given $\left[\left[ X_c \right] \right]_{c}$ and $X_{b_1}$, the multiplication between the ciphertext and plaintext is $\left[\left[ X_c \right] \right]_{c} \bigotimes  X_{b_1} = \left[\left[ X_c * X_{b_1} \right] \right]_{c}$.
    \item\textbf{Decryption}: $\mathbf{C}$ uses $sk_c$ to decrypt the ciphertext and can get the  plaintext $X_c$ ($X_c=\textbf{Dec}(\left[\left[ X_c \right] \right]_{c},sk_c)$).
\end{enumerate}
We call this probabilistic asymmetric encryption scheme for restricted computation (addition and multiplication) over ciphertexts partially homomorphic encryption (PHE). 
In this paper, we utilize the Paillier cryptosystem~\cite{Paillier99}.

\subsection{Generalized Linear Models}
\label{glm}
Generalized linear models (GLMs) are flexible generalizations of linear regression. The GLM generalizes linear regression by allowing the linear model to be related to the response variable ($Y$) via a link function and by allowing the magnitude of the variance of each measurement to be a function of its predicted value.
Each kind of GLM consists of the following three elements:
\begin{enumerate}
    \item An exponential family of probability distributions that $E(Y|X)$ is assumed to satisfy.
    \item A linear predictor $ \eta =WX $.
    \item A link function $g$ such that $E(Y|X)=\mu =g^{-1}(\eta )$.
\end{enumerate}
The maximum likelihood estimation (MLE)  using iteratively gradient descent method is commonly applied to solve the weight parameters ($W$) of GLMs. 
For Example:

$\mathbf{Logistic\; regression}$ is actually a binary classification model and is widely used in industry because of its simplicity and interpretability. It 
assumes that $E(Y|X)$ satisfies Bernoulli distribution, and the link function is $g=\frac{\eta}{1-\eta}$. LR can find a direct relationship between the classification probability and the input vector ($WX$)
by operating the Sigmoid function.
The loss of LR can be calculated through MLE,
\begin{equation}
\label{eq:lr_loss}
    loss_{LR} = \frac{1}{m}  \sum_{i=1}^m ln(1+e^{-YWX}),
\end{equation}
where $m$ is the sample size, and $Y\in$\{-1,1\} is the data label.

With the formula of loss, we can calculate its gradient,
 and approximate the gradient with MacLaurin expansion to avoid non-linear calculations:
\begin{equation}
\label{eq:lr_grad}
    g_{LR} = X^T*\frac{1}{m}(0.25WX - 0.5Y),
\end{equation}
where $T$ means the transposition of a matrix and $*$ is the matrix multiplication operation.

$\mathbf{Poisson\; regression}$  assumes that $E(Y|X)$  has a Poisson distribution and usually adopts Log function ($g=\ln(\eta)$) as its link function.
PR is used to represent counts of rare independent events which happen at random but at a fixed rate, such as the number of claims in insurance policies in a certain
period of time, and the number of purchases a user makes after being shown online advertisements, etc. Also with the MLE, the formula of its loss can be written as:
\begin{equation}
\label{eq:poisson_loss}
    loss_{PR} = \frac{1}{m}  \sum_{i=1}^m (YWX-e^{WX} - ln(Y!))
\end{equation}
and the gradient can be calculated by the following equation:
\begin{equation}
\label{eq:poisson_grad}
    g_{PR} = X^T*\frac{1}{m}(e^{WX}-Y).
\end{equation}
\section{Proposed framework}

In this section, we first introduce the core of ideology and the main architecture of our proposed framework.
Then we show readers how to apply our framework into GLMs (like LR, PR) in federated learning scenarios.
In the following, we show that it's easy for our framework to be extended to multiple parties (three or more).
At last, we give a security analysis of our framework.

\subsection{Ideology and Architecture}
\label{2p}
The secret sharing is usually used in MPC-based schemes, which leads to resource and communication consumption because the original data is split and then all shared.
In our schemes, instead of sharing original data, we split intermediate results of GLMs training (e.g., $W_{p}X_{p}$) only, which leads to a dramatic drop in communication.

There are four main protocols in our framework, i.e., 
Secret sharing protocol, Secure gradient-operatorcomputing~protocol,
Secure gradient computing protocol, and Secure loss computing~protocol.
In this section, we only consider two-party situation (party $\mathbf{C}$ and $\mathbf{B_1}$) 
and will explain how to extend our framework to multiple parties in Section \ref{mp}.

$\mathbf{Secret\; sharing\; protocol}$ 
is used to securely split data into shares. Then party $\mathbf{C}$ and $\mathbf{B_1}$ can get the shares. It can be 
achieved by existing MPC protocols like SPDZ, secureML, etc. Protocol \ref{alg:ss-protocol} is an example of secret sharing.

\begin{protocol}[tb]
\caption{Secret sharing protocol}
\label{alg:ss-protocol}
\textbf{Input}: a vector data $Z$, party $P_0$, $Z$ is held by $P_0$, $P_0$ may be $\mathbf{C}$ or $\mathbf{B_{1}}$\\
\textbf{Output}: $\langle Z\rangle _{c}$ for $\mathbf{C}$ and $\langle Z\rangle _{b_{1}}$ for $\mathbf{B_{1}}$, 
                 and $\langle Z\rangle _{c}$+ $\langle Z\rangle _{b_{1}}$=$Z$

\begin{algorithmic}[1] 
\STATE $P_1$= $\{\mathbf{C}$,$\mathbf{B_{1}}\}$-$\{P_0\}$
\STATE $P_0$ locally generates a share $\langle Z \rangle _{p_0}$
\STATE $P_0$ calculates $\langle Z\rangle _{p_1}$ =  $Z$ - $\langle Z\rangle _{p_0}$
\STATE $P_0$ sends $\langle Z\rangle _{p_1}$ to $\mathbf{P_1}$
\STATE \textbf{return} $\langle Z\rangle _{c}$ for $\mathbf{C}$ and $\langle Z\rangle _{b_{1}}$ for $\mathbf{b_{1}}$
\end{algorithmic}
\end{protocol}

$\mathbf{Secure\; gradient \mbox{-} operator\; computing\; protocol.}$ 
Gradient descent is the main method for solving machine learning tasks. As introduced in Section \ref{glm}, the gradient of GLMs can be formalized as follows:
\begin{equation}
\label{eq:grad}
   g = X^T*d,
\end{equation}
where $g$ is the gradient, $X$ is the feature data,
and we define $d$ as gradient-operator. The calculation formula of $d$ varies with different models.
We will show this in the next subsection. Through this formula, we can calculate shares of $d$ for different parties using the MPC method based on SS.
Eventually party $\mathbf{C}$ gets $\langle d \rangle_c$ and party $\mathbf{B_1}$ gets $\langle d\rangle_{b_1}$ (see Protocol~\ref{alg:sgo-protocol}).

\begin{protocol}[tb]
\caption{Secure gradient-operator computing protocol}
\label{alg:sgo-protocol}
\textbf{Input}: $\langle Z\rangle_{c}$'s and $\langle Z\rangle _{b_{1}}$'s calculated by Protocol \ref{alg:ss-protocol}, 
$Z$  could be the share of $WX$, share of $Y$, etc.\\
\textbf{Output}: $\langle d \rangle _{c}$ for $\mathbf{C}$ and $\langle d \rangle _{b_{1}}$ for $\mathbf{B_{1}}$

\begin{algorithmic}[1] 
\STATE According to eq~(\ref{eq:lr_grad}) and eq~(\ref{eq:poisson_grad}), $\mathbf{C}$ and $\mathbf{B_{1}}$ calculate $\langle d \rangle _{c}$  and $\langle d \rangle _{b_{1}}$ base on MPC method separately
\STATE \textbf{return} $\langle d\rangle _{c}$ for $\mathbf{C}$ and $\langle d\rangle _{b_{1}}$ for $\mathbf{B_{1}}$
\end{algorithmic}
\end{protocol}
$\mathbf{Secure\; gradient\; computing\; protocol.}$ 
The gradient of GLMs can be calculated by eq~(\ref{eq:grad}). However, the gradient-operator $d$ is shared over party $\mathbf{C}$ and $\mathbf{B_{1}}$. 
So for the gradient of one party (e.g., $\mathbf{C}$), its one share ($\langle g \rangle_c$) can be locally calculated by eq~(\ref{eq:grad}), 
while another share ($\langle g \rangle_{b_1}$) can not be computed directly, because party $\mathbf{B_{i}}$ can neither get $X_c$ directly nor send his $\langle g \rangle_{b_1}$ to  $\mathbf{C}$. 
In order to securely compute the gradient,  we introduce homomorphic encryption in this protocol (details in Protocol~\ref{alg:sg-protocol}). 
After the gradient is calculated, model coefficients can be updated using eq~(\ref{eq:grad1}) locally.  
\begin{equation}
\label{eq:grad1}
    W = W - \alpha g,
\end{equation}
where $\alpha$ is the learning rate.

\begin{protocol}[tb]
\caption{Secure gradient computing protocol}
\label{alg:sg-protocol}
\textbf{Input}: $\langle d\rangle_{c}$ and $\langle d\rangle _{b_{1}}$ calculated by Protocol \ref{alg:sgo-protocol}, 
feature data matrix $X_{p_0}$, party $P_0$, $X_{p_0}$ is held by $P_0$, $P_0$ may be $\mathbf{C}$ or $\mathbf{B_{1}}$.\\
\textbf{Output}: $g_{c}$ for $\mathbf{C}$ if $P_0$ is $\mathbf{C}$ or $g_{b_1}$ for $\mathbf{B_1}$ if $P_0$ is $\mathbf{B_1}$

\begin{algorithmic}[1] 
\STATE $P_1$= $\{\mathbf{C}$,$\mathbf{B_{1}}\}$-$\{P_0\}$
\STATE $P_0$ locally calculates the share of gradient using eq~(\ref{eq:grad}), i.e. $\langle g_{p_0}\rangle_{p_0} = X_{p_0}^T*d_{p_0}$
\STATE $P_1$ encrypts $\langle d\rangle_{p_1}$ using $pk_{p_1}$ and sends $\left[\left[ \langle d\rangle_{p_1} \right]\right]_{p_1}$ to $P_0$
\STATE $P_0$ locally calculates encrypted share of gradient using eq(\ref{eq:grad}), 
             i.e., $\left[\left[\langle g_{p_0}\rangle_{p_1}\right]\right]_{p_1} = X_{p_0}^T*\left[\left[ \langle d\rangle_{p_1} \right]\right]_{p_1}$
\STATE $P_0$ locally generates random noise $R_{p_0}$
\STATE $P_0$ locally masks encrypted share of $g_{p_0}$ with 
             $\left[\left[\langle g_{p_0}\rangle_{p_1}\right]\right]_{p_1} = \left[\left[\langle g_{p_0}\rangle_{p_0}\right]\right]_{p_1}-R_{p_0}$
             and sends $\left[\left[\langle g_{p_0}\rangle_{p_1}\right]\right]_{p_1}$ to $P_1$,   
\STATE $P_1$ gets $\langle g_{p_0}\rangle_{p_1}$ by decrypting $\left[\left[\langle g_{p_0}\rangle_{p_1}\right]\right]_{p_1}$ with $sk_{p_1}$, 
             and sends $\langle g_{p_0}\rangle_{p_1}$ to $P_0$
\STATE $P_0$ calculates the gradient of its coefficients with $g_{p_0}$ =  $\langle g_{p_0}\rangle_{p_1} + \langle g_{p_0}\rangle_{p_0}+R_{p_0}$
\STATE \textbf{return} $g_{c}$ for $\mathbf{C}$ or $g_{b_{1}}$ for $\mathbf{B_{1}}$
\end{algorithmic}
\end{protocol}

 $\mathbf{Secure\; loss\; computing\; protocol.}$ 
After the model coefficients are updated, it is always needed to see how well they fit the label and whether the training process needs to stop.
Loss value is the most direct and commonly used indicator to quantify the model performance. For example, when the loss is less than a certain threshold value, the model is considered to fit well. 
There are different common forms of loss for different GLMs. In the federated learning scenarios, it's also needed to calculate loss securely.
Similar to Protocol~\ref{alg:sgo-protocol}, shares of loss can be calculated based on the secret-sharing intermediate results (see Protocol~\ref{alg:sl-protocol}).
The difference is that loss value needs revealing to party $\mathbf{C}$ for evaluating the model at last.

\begin{protocol}[tb]
\caption{Secure loss computing protocol}
\label{alg:sl-protocol}
\textbf{Input}: $\langle Z\rangle_{c}$'s and $\langle Z\rangle _{b_{1}}$'s calculated by Protocol~\ref{alg:ss-protocol}, 
$Z$ could be the share of $WX$, share of $Y$, etc.\\
\textbf{Output}: $loss$ for $\mathbf{C}$
\begin{algorithmic}[1] 
\STATE According to eq~(\ref{eq:lr_loss}) and eq~(\ref{eq:poisson_loss}), $\mathbf{C}$ and $\mathbf{B_{1}}$ calculate $\langle loss \rangle _{c}$  and $\langle loss \rangle _{b_{1}}$ based on MPC method separately

\STATE $\mathbf{B_1}$ sends $\langle loss \rangle _{b_{1}}$ to $\mathbf{C}$
\STATE $\mathbf{C}$ reveals loss by $loss = \langle loss \rangle _{c} + \langle loss \rangle _{b_{1}}$
\STATE \textbf{return} $loss$ for $\mathbf{C}$
\end{algorithmic}
\end{protocol}

Please refer to Algorithm~\ref{alg:my_alg} for specific usage of these four protocols.

\subsection{ PR and LR in VFL }
In this section, we introduce how to apply our framework to GLMs.
The differences among different GLMs only exist in Protocol $\ref{alg:sgo-protocol}$ and Protocol $\ref{alg:sl-protocol}$. 
Specifically, we will take the implementations of LR and PR as examples.

$\mathbf{Logistic\; regression.}$
According to eq~(\ref{eq:lr_grad}), the gradient-operator of LR in Protocol \ref{alg:sgo-protocol} is:
\begin{equation}
\label{eq:lr_grad_operator}
    d_{LR} = \frac{1}{m}(0.25WX - 0.5Y).
\end{equation}
On the other hand, $Secure\; loss\; computing\; protocol$ (i.e., Protocol~\ref{alg:sl-protocol}) can be implemented according to eq~(\ref{eq:lr_loss}).
On the basis of the calculation formulas of loss and gradient, party $\mathbf{C}$ and $\mathbf{B_1}$ both need to share $WX$ in Protocol \ref{alg:ss-protocol}, 
and $\mathbf{C}$ needs to share label $Y$ additionally.

$\mathbf{Poisson\; regression}.$
Similarly, the gradient-operator of PR can be calculated by following according to eq~(\ref{eq:poisson_grad}):
\begin{equation}
\label{eq:pr_d}
    d_{PR} = \frac{1}{m}(e^{WX}-Y).
\end{equation} 
To avoid non-linear operations when calculating gradient-operator and loss of PR based on MPC, shares of
$e^{WX}$ are also required in Protocol~\ref{alg:ss-protocol} in addition to $WX$ and $Y$.

Please note that although we introduce our framework in the situations of LR and PR, the framework is also suitable for other
GLMs, eg, Liner, Gamma, Tweedie regression, etc.

\subsection{ Multi-party VFL}
\label{mp}
It is easy to extend our framework to multiple parties, as shown in our whole framework (Algorithm~\ref{alg:my_alg}). We will describe more details below.

When it comes to multiple parties, we need to first select two parties as the computing party (hereafter, CP) that holds the secret shares and computes based on the shares. 
To prevent collusion between the two CPs, we can select two different parties randomly in each iteration. 
In order to be consistent with the previous description in Section \ref{2p}, we select party $\mathbf{C}$ and $\mathbf{B_{1}}$ 
as CP all the time in Algorithm~\ref{alg:my_alg}.

From Algorithm~\ref{alg:my_alg} we can see that, by selecting two parties as CPs, 
Protocol \ref{alg:ss-protocol} doesn't need to change for CPs. 
Other parties just need to first locally generate shares of their own data and then send shares to the two CPs separately.

Further, Protocol~\ref{alg:sgo-protocol} and \ref{alg:sl-protocol} even don't need to change because only the CPs hold shares,
and they can complete the gradient-operator and loss computing tasks by themselves. 
While CPs may have to send the encrypted gradient-operator to other participants after the computing is done in Protocol~\ref{alg:sgo-protocol}.

The protocol that needs to make a little big change is $Secure~gradient~computing~protocol$ (i.e., Protocol \ref{alg:sg-protocol}). Similar to Protocol \ref{alg:ss-protocol}, 
gradient's calculation of the CPs needs no change. 
While other parties need to get the two encrypted shares of gradient-operator from CPs first, since those shares are calculated in CPs.
Then other parties can locally calculate the two encrypted shares of gradient by eq~(\ref{eq:grad}).
To protect the gradient, other parties also need to mask the two shares with noises before sending them to CPs for decryption. After receiving the decrypted shares, 
they could reveal the true gradient by removing the noises.


\begin{algorithm*}[tb]
\caption{EFMVFL: An Efficient and Flexible Multi-party Vertical Federated Learning without a Third Party}
\label{alg:my_alg}
\textbf{Input}:  feature data $X_{p}$, label data $Y$, HE key pairs  $\{pk_p,sk_p\}$, 
                party $P$ may be $\mathbf{C}$ or $\mathbf{B_{i}}$, learning rate $\alpha$, max iteration $T$, loss threshold $L$, stop flag $falg$\\
\textbf{Output}: $W_{p}$ for $P$
\begin{algorithmic}[1] 
\STATE Let $t=0$
\STATE Initialize $W_p$ as zero vector, set $flag=false$
\WHILE{$t<T$ and not $flag$ }
\STATE Select two parties from $\{\mathbf{C},\mathbf{B_{i}}\}$ as the comuting party, such as $\mathbf{C}$ and $\mathbf{B_{1}}$
\STATE $P$ locally calculates $Z$'s, $Z$ could be $W_pX_p$, $e^{W_pX_p}$, $Y$, etc., and $Y$ is held by $\mathbf{C}$
\IF {$P$ is comuting party}
\STATE Do secret sharing with Protocol \ref{alg:ss-protocol}.
\STATE Recive shares from other parties
\STATE Do secure gradient-operator computing with Protocol \ref{alg:sgo-protocol}.
\STATE Do secure gradient computing with Protocol \ref{alg:sg-protocol}, get gradient $g_p$.
\STATE Send encrypted gradient-operator to other parties
\STATE Recive encrypted gradient from other parties and decrypt and send back to other parties
\STATE Do secure loss computing with Protocol \ref{alg:sl-protocol}.
\ELSE
\STATE Generate two shares of $Z$'s as $\langle Z \rangle _{c}$'s+$\langle Z \rangle _{b_1}$'s = $Z$'s
\STATE Send shares $\langle Z \rangle _{c}$'s to $\mathbf{C}$, $\langle Z \rangle _{b_1}$'s to $\mathbf{B_{1}}$.
\STATE Recive encrypted gradient-operator $\left[\left[\langle d\rangle _{c} \right]\right]_{c}$,
       $\left[\left[\langle d\rangle_{b_{1}}\right]\right]_{b_1}$  form $\mathbf{C}, \mathbf{B_{1}}$ separately.
\STATE Calculate two encrypted shares of $P$'s gradient
       $\left[\left[\langle g_{p}\rangle_{c}\right]\right]_{c}, \left[\left[\langle g_{p}\rangle_{b_1}\right]\right]_{b_1}$ using eq~(\ref{eq:grad})
\STATE Generate two random noise $R_c, R_{b_1}$,
       Send $\left[\left[\langle g_{p}\rangle_{c}\right]\right]_{c}-R_c$ to $\mathbf{C}$, $\left[\left[\langle g_{p}\rangle_{b_1}\right]\right]_{b_1}-R_{b_1}$ 
       to $\mathbf{B_{1}}$.
\STATE Recive $\langle g_{p}\rangle_{c}, \langle g_{p}\rangle_{b_1}$ form $\mathbf{C}, \mathbf{B_{1}}$ separately
\STATE Calculate the gradient of $P$'s coefficients with $g_{p}$ =  $\langle g_{c}\rangle_{c} + \langle g_{b_1}\rangle_{b_1}+R_{c}+R_{b_1}$
\ENDIF
\STATE Update model coefficients $W_p$ using eq~(\ref{eq:grad1})
\IF {$P$ is $\mathbf{C}$ }
\IF {$loss < L$}
\STATE Set $flag = true$
\ENDIF
\STATE Send $flag$ to other parties
\ELSE
\STATE get $flag$ from  $\mathbf{C}$
\ENDIF
\STATE $t= t + 1$
\ENDWHILE
\STATE \textbf{return} $W_{p}$ for $P$
\end{algorithmic}
\end{algorithm*}

\subsection{Security Analysis}



\newtheorem{lemma}{Lemma}

\newcommand{\X}{\mathbf{X}}
\newcommand{\Z}{\mathbb{Z}}
\newcommand{\w}{\mathbf{w}}
\newcommand{\y}{\mathbf{y}}
\newcommand{\g}{\mathbf{g}}
\renewcommand{\d}{\mathbf{d}}
\newcommand{\A}{\mathcal{A}}

The security of our multi-party model architecture only depends on the security of two-party model architecture. For the two-party model, we use the same security model and architecture as \cite{ZWW21}. 
Following is the detailed security analysis.

\begin{theorem}
\label{th}
Let $\{\g_i \in \Z_q^{m_1}\}_{i\in[T]}$, $\X_1\in\Z_q^{n\times m_1 }$, $\X_2 \in \Z_q^{n\times m_2 }$, $\{\d_i \in \Z_q^n\}_{i\in[T]}$,  $\{\w_i \in \Z_q^{m_2}\}_{i\in [T]}$, where n, m and q integers and T is the number of samples.  For any probabilistic polynomial-time adversary $\A$,  given some $\{\g_i \in \Z_q^{m_1} \}_{i\in[T]}$ and $\X_1 \in \Z_q^{n\times m_1}$ which satisfy that $\g_i = \X_1^T \d_i$ for all $i\in [T]$, the necessary condition for $\X_2$ and $\{\w_i\}_{i\in[T]}$ which satisfy that $\d_i = \X_2 \w_i$ for all $i \in [T]$ can not to be accurately calculate if $n > m_1$ or $n \leq \min\{m_1,m_2\}$ or $m_2<n\leq m_1$, $T\leq \frac{n\times m_2}{n-m_2}$ .
	
\end{theorem}
Before proving \textbf{Theorem} \ref{th}, we prove two necessary lemmas to prepare for the subsequent proof.

\begin{lemma}
\label{l1}
Let $\X \in \Z_q^{n\times m}$, $\{\w_i \in \Z_q^m\}_{i\in [T]}$, where n, m and q are integers and T is the number of samples. For any probabilistic polynomial-time adversary $\A$, given some $\{\y_i \in \Z_q^n \}_{i\in [T]}$ which satisfy that $\X \cdot \w_i =\y_i$ for $i\in[T]$, the necessary condition for $\X$ and $\{\w_i\}_{i\in[T]}$ can not to be accurately calculated is $n \leq m$ or $n > m$ and $T \leq \frac{n\times m}{n-m}$.
\end{lemma}

\begin{proof}
When the adversary $\A$ gets a vector $\y_i$, which means that $\A$ gets $n$ equations and $n \times m + m$ unknowns. Similarly, through $T$ samples, $\A$ can obtain $n \times T$ equations and $(n + T)\times m$ unknowns. There are two cases in this situation.
\begin{description}
\item[case 1: $n \leq m$.] In this case, we can get $nT < (n+T)m$ which means that the number of unknowns is large than the number of equations. At this time, $\A$ cannot calculate the precise information of $\X$ and $\{\w_i\}_{i\in[T]}$. 
\item[case 2: $n > m$.] In this case, we just need to set $T \leq \frac{n\times m}{n-m}$, then we can get $nT < (n+T)m$. which means that the number of unknowns is large than the number of equations. At this time, $\A$ cannot calculate the precise information of $\X$ and $\{\w_i\}_{i\in[T]}$.
\end{description}

In conclusion, the necessary condition for $\X$ and $\{\w_i\}_{i\in[T]}$ can not to be accurately calculated is $n \leq m$ or $n > m$ and $T \leq \frac{n\times m}{n-m}$.
\end{proof}

\begin{lemma}
\label{l2}
Let $\X \in \Z_q^{n\times m}$, $\d \in \Z_q^n$ where n, m and q are integers. For any probabilistic polynomial-time adversary $\A$, given $\y \in \Z_q^m$ and $\X \in \Z_q^{n\times m}$ which satisfy that $\y = \X^T \d$, the necessary condition for $\d$ to be accurately calculated is $n \leq m$. 
\end{lemma}

\begin{proof}
When the adversary $\A$ gets $\y \in \Z_q^m$ and $\X \in \Z_q^{n\times m}$, which means that $\A$ gets $m$ equations and $n$ unknowns. The necessary condition for $\d$ to be accurately calculated is $n \leq m$. 
\end{proof}

Now we give a formal proof of \textbf{Theorem} \ref{th}.

\begin{proof}
  If we want to accurately calculate $\X_2$ and $\{\w_i\}_{i\in[T]}$, then we must first accurately calculate the value of $\d_i$. 
  According to this condition, we will discuss this issue in three situations below.
  \begin{description}
  	\item [case 1: $n > m_1$] By \textbf{Lemma} \ref{l2}, the necessary condition for $\d_i$ to be accurately calculated is $n \leq m_1$. In this case,  $\d_i$ can not to be accurately calculated, and then $\X_2$ and $\{\w_i\}_{i\in[T]}$ can not to be accurately calculated.

  	\item [case 2: $n \leq \min\{m_1,m_2\}$] In this case, we can accurately calculate $\d_i$ by \textbf{Lemma} \ref{l2}. However, in this situation, through \textbf{Lemma} \ref{l1}, any probabilistic polynomial-time adversary $\A$ can not accurately calculate $\X_2$ and $\{\w_i\}_{i\in[T]}$. 

  	\item [case 3: $m_2<n\leq m_1$, $T\leq \frac{n\times m_2}{n-m_2}$] In this case, we can accurately calculate $\d_i$ by \textbf{Lemma} \ref{l2}. However, in this situation, through \textbf{Lemma} \ref{l1}, any probabilistic polynomial-time adversary $\A$ can not accurately calculate $\X_2$ and $\{\w_i\}_{i\in[T]}$.

  \end{description}  	
 \end{proof} 

In the following, we will give a formal theorem that \textbf{Protocol \ref{alg:ss-protocol}} is secure against semi-honest adversaries. 
\begin{theorem}
\label{theorem for A1}
Assume that $ \langle Z \rangle_{p0} $ is generated by a secure pseudo-random number generator (PRNG). Then \textbf{Protocol 1} is secure in semi-honest model.  
\end{theorem}

\begin{proof}
We prove that $Z - \langle Z \rangle_{p0}$ is computationally indistinguishability with random number. $ \langle Z \rangle_{p0} $ is computationally indistinguishability with random number because it is generated by a secure PRNG. Then $Z - \langle Z \rangle_{p0}$ is computationally indistinguishability with random number.
\end{proof}

In the following, we will give a formal theorem that \textbf{Protocol \ref{alg:sgo-protocol}} and \textbf{Protocol \ref{alg:sl-protocol}} are secure against semi-honest adversaries. 
\begin{theorem}
\label{theorem for A2}
Assume that there is a secure MPC protocol under semi-honest adversary model. Then \textbf{Protocol \ref{alg:sgo-protocol}} and \textbf{Protocol \ref{alg:sl-protocol}} are secure in semi-honest model.
\end{theorem}

\begin{proof}
Obviously, the security of \textbf{Protocol \ref{alg:sgo-protocol},\ref{alg:sl-protocol}} is directly dependent on the security of MPC protocol. There are many MPC protocols which are secure under semi-honest adversary, just like \cite{ABY3,SecureNN}.
\end{proof}

In the following, we will give a formal theorem that \textbf{Protocol \ref{alg:sg-protocol}} is secure against semi-honest adversaries.
\begin{theorem}
\label{theorem for A3}
Assume that the additively homomorphic encryption protocol $\Pi = (KeyGen, Enc, Dec)$ is indistinguishable under chosen-plaintext attacks. Then \textbf{Protocol \ref{alg:sg-protocol}} is secure in semi-honest model.
\end{theorem}

\begin{proof}
  In line 1-4, the data is calculated locally and sent to other participants in the form of ciphertext. The security of these lines is dependent on the homomorphic encryption protocol $\Pi$. In line 5-7, we use the same technology which has been proposed in \cite{ZWW21}'s \textbf{protocol 2}. They give a detailed proof in Appendix B.1, which we will not repeat it here. In line 8-9, $P_0$ will get some gradient $g$. Generally speaking, $n$ is much larger than $m$. Through \textbf{Theorem} \ref{th}, we can know that $P_0$ could not calculate other party's feature data matrix and linear model coefficients.  
\end{proof}

In the following, we will give a formal theorem that \textbf{Algorithm \ref{alg:my_alg}} is secure against semi-honest adversaries.
\begin{theorem}
\label{theorem for A5}
Assume that the additively homomorphic encryption protocol $\Pi = (KeyGen, Enc, Dec)$ is indistinguishable under chosen-plaintext attacks and there is a secure MPC protocol under semi-honest adversary model. Then \textbf{Algorithm \ref{alg:my_alg}} is secure in semi-honest model.
\end{theorem}

\begin{proof}
 \textbf{Algorithm \ref{alg:my_alg}} comprehensively uses the structure of \textbf{Protocol 1-4}. The proof of \textbf{Theorem} \ref{theorem for A5} can reuse the proofs of \textbf{Theorem} \ref{theorem for A1} \ref{theorem for A2} \ref{theorem for A3}.
\end{proof}

\section{Implementations and Assessments}

In this section, we implement adequate experiments to show that our framework is effective and efficient with less communication, 
applicable to GLMs, and easy to scale to multi-party modeling.

\subsection{Dataset}
Our experiments are based on the following open-source datasets. 
We vertically split both datasets into two parts as Fate
\footnote{\url{https://github.com/FederatedAI/FATE}} does, corresponding to party C and B1.
In the multi-party case, we easily copy the data
of party B1 to the new party.

$\mathbf{Default~of~credit~card~clients~Dataset}$ 
\footnote{\url{https://archive.ics.uci.edu/ml/datasets/default+of+credit+card+clients}}
consists of 30 thousand samples with 24 attributes that aimed at the case of customers' default payments in Taiwan, 
which is used in our LR experiments. 

$\mathbf{Dvisits~Dataset}$
\footnote{\url{https://www.rdocumentation.org/packages/faraway/versions/1.0.7/topics/dvisits}}
is used in the PR scenario, which comes from the Australian Health Survey of 1977-1978 and consists of 5190 single
 adults with 19 features.


\subsection{Setting}
All our experiments are run on Linux servers with 32 Intel(R) Xeon(R) CPU E5-2640 v2@2.00GHz and 128GB RAM. For each server, the used
CPU resources and network bandwidth are limited to 16 cores and 1000Mbps, respectively. For parameters other
than hardwares, the Paillier HE key length, max iteration, threshold, learning rate of LR and PR are set to 1024, 30, 1e-4, 0.15 and 0.1, respectively. 
We set the ratio between training and test set to 7:3

\subsection{Experiments and result}
We compare our framework with those methods with (TP-LR~\cite{HELR2018}, TP-PR inspired by~\cite{HEPR2017})  and without (SS-LR~\cite{SSLR2020}, SS-HE-LR~\cite{ZWW21})
a third party. Here, the TP-LR and TP-PR are based on HE; the SS-LR is based on SS and the SS-HE-LR is based on both SS and HE methods.
In addition to the SS-HE-LR that has been implemented by FATE, we implement other methods by ourselves.



\begin{figure}[h]
\centering
\includegraphics[width=8cm,height=5cm]{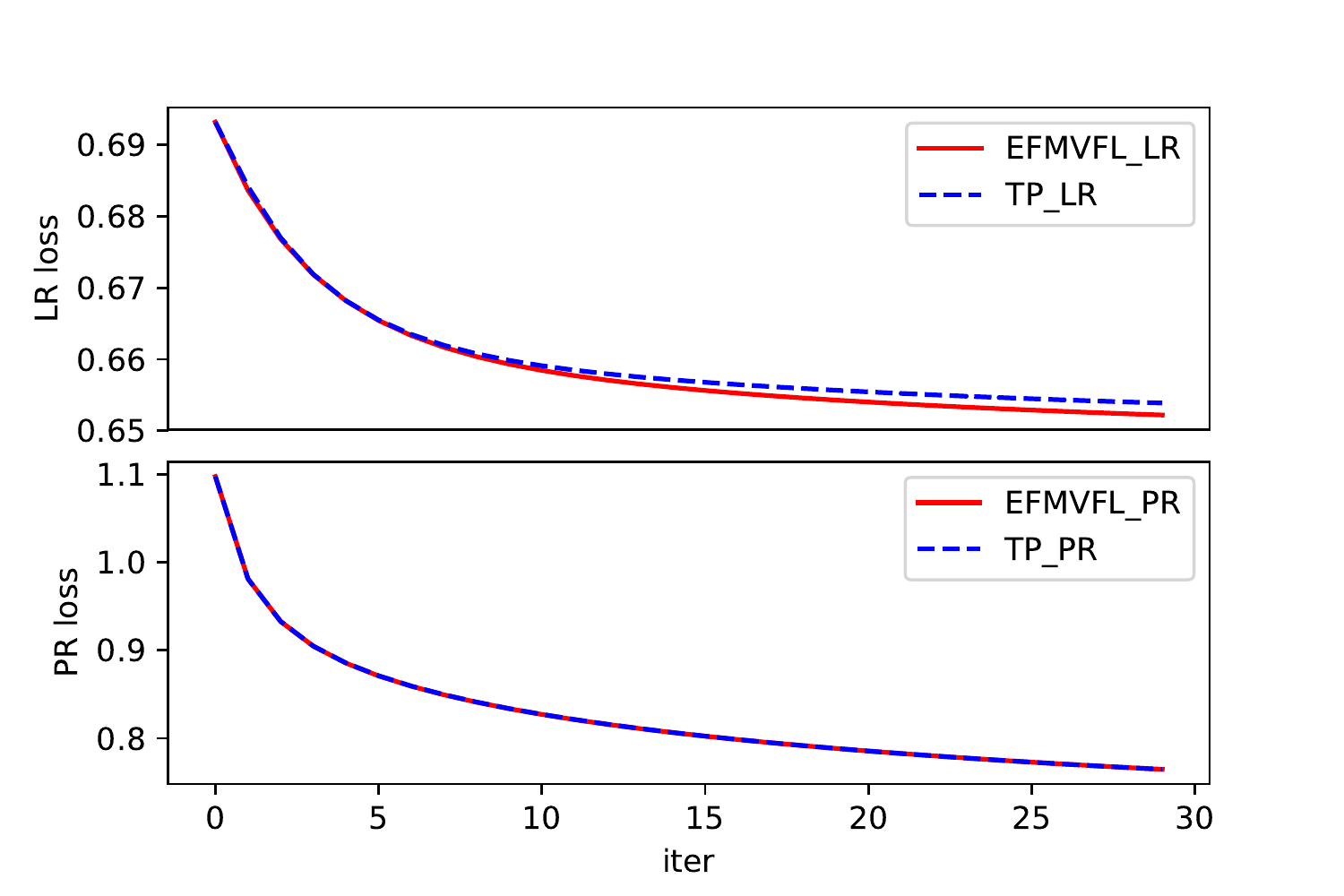}
\caption{Training loss curve of LR (upper) and PR (lower)}
\label{fig:loss}
\end{figure}

From the training loss curve in Figure \ref{fig:loss}, we show that the losses of our framework (red solid lines) are almost identical to those methods with a
third party (blue dashed lines). Note that the difference between loss curves in the upper panel is because
the loss used in TP-LR is a Taylor approximation of our method.
\begin{table}
\centering
\begin{tabular}{lllll}
\hline
framework  & auc & ks & comm  & runtime \\
\hline
TP-LR       & 0.712 & 0.371  & \textbf{14.20mb} &  34.79s     \\
SS-LR       & 0.719 & 0.363  & 181.8mb & 71.05s     \\
SS-HE-LR    & 0.702 & 0.367  & 85.30mb &  37.6s    \\
EFMVFL-LR   & 0.712 & 0.372  & \textbf{26.45mb} & \textbf{23.29s}    \\
\hline
\end{tabular}
\caption{LR results on test set}
\label{tab:LR}
\end{table}
    
\begin{table}
\centering
\begin{tabular}{lllll}
\hline
framework  & mae &rmse &comm & runtime \\
\hline
TP-PR       & 0.571 & 0.834 & \textbf{4.27mb} & 12.44s     \\
EFMVFL-PR  & 0.571 & 0.834 & \textbf{5.60mb} & \textbf{10.78s}     \\
\hline
\end{tabular}
\caption{PR results on test set}
\label{tab:PR}
\end{table}
Table \ref{tab:LR} and Table \ref{tab:PR} provide more details about comparison results for these methods. Please note that all these results are measured on test set in the case of 2 parties.
Since our framework needs only one product between plaintext matrix and ciphertext vector for each party in each iteration (see line 4 of Protocol~\ref{alg:sg-protocol}), it shows less communication consumption and is the most efficient, as expected.

What's more, the communication of our framework is the least one among methods that utilize SS, while a little more than the method that needs a third party. 
The main reason is that compared to SS-based methods that have to share all of original data, our framework only needs to share a vector (see Section~\ref{2p}).

Figure \ref{fig:party} shows the communication and runtime of our framework in LR scenario as a function of the number of participants, and it is similar in PR.
Here, the star points represent the experimental measurements. 
In the lower panel, for clarity, we fit a straight line to show that the framework communication increases linearly with the number of participants. 
In the upper panel, as the number of participants changes from 2 to 3, we find the runtime increases suddenly and then flattens out. 
This is because when it comes to multiple parties, there would be 2 cipher product operations for parties that are not the MPC computing party (see Algorithm~\ref{alg:my_alg}).

\begin{figure}[!h]
\centering
\includegraphics[width=8cm,height=5cm]{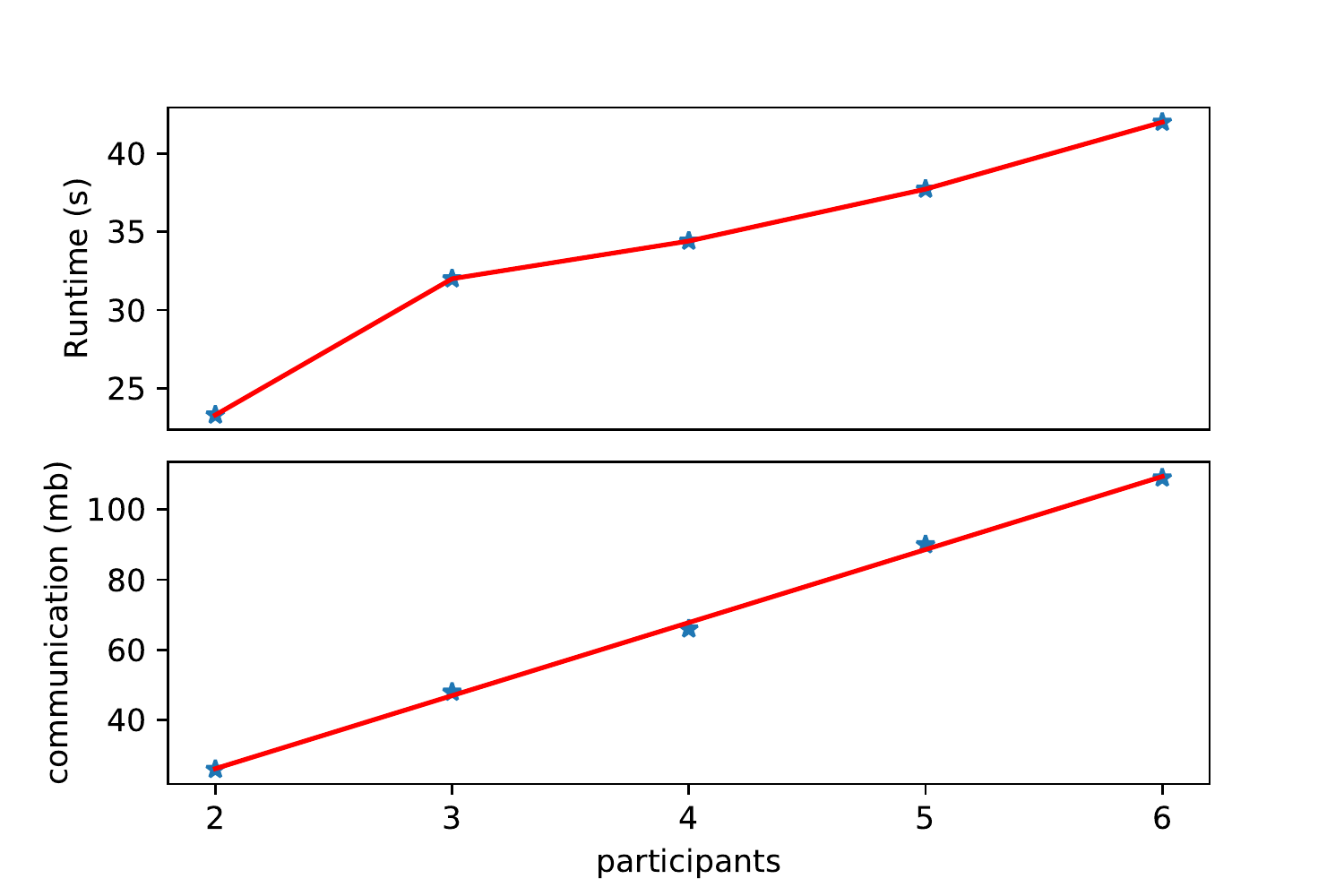}
\caption{Variation curve of runtime (upper) and communication (lower) as the number of participants increases}
\label{fig:party}
\end{figure}


\section*{Conclusion}
In this paper, we present an Efficient and Flexible Multi-Party Vertical Federated Learning framework (EFVFL) that does not require a third party by combining SS and HE.
The framework is applicable to many kinds of generalized linear regression models and has been shown in logistic regression and Poisson regression scenarios. 
Through theoretical analysis and comparison with some recent popular FL works, we show that EFVFL is secure, effective, and more efficient with less communication overhead.
What's more, our framework is scalable to multi-party modeling, and experiments show that runtime and communication both grow almost linearly as the number of participants increases.
In the future, we will expand our framework to more machine learning algorithms.
\section*{Acknowledgments}
We thank Tianxiang Mao and Renzhang Liu for their professional advice to this work.

\bibliographystyle{named}
\bibliography{ijcai22}

\end{document}